\documentclass{article}
\usepackage[utf8]{inputenc}
\usepackage{graphicx}
\usepackage[margin=1in]{geometry}
\usepackage{amsmath, amssymb}
\usepackage{url}
\usepackage[colorlinks=true,
            linkcolor=blue,
            citecolor=red,
            urlcolor=magenta]{hyperref}
\usepackage[ruled,vlined,linesnumbered,algo2e]{algorithm2e}
\usepackage{amsthm}
\usepackage{cite}

\newtheorem{theorem}{Theorem}

\newtheorem{proposition}[theorem]{Proposition}

\newtheorem{definition}[theorem]{Definition}
\newtheorem{remark}[theorem]{Remark}

\usepackage{rotating} 
\usepackage{array}     
\usepackage{longtable} 
\usepackage{booktabs}
\usepackage{siunitx}

\newcommand{\atoms}{\mathcal{A}}

\newcommand{\OmegaSet}{\Omega}
\newcommand{\Sem}{\mu}

\newcommand{\Bel}{b_{\boldsymbol{\theta}}}
\newcommand{\LogicFn}{l}
\newcommand{\Measure}{m}

\newcommand{\B}{\mathbb{B}}
\newcommand{\E}{\mathbb{E}}
\newcommand{\Ind}[1]{\mathbb{I}\!\left[#1\right]}

\title{From Hybrid Mechanistic--Data-Driven Modeling Toward Neuro-Symbolic AI: What, Why, and How}

\author{
    Moein E. Samadi$^{1,2}$, Andreas Schuppert$^{*,1,2}$ \\
    \\
    \small
    $^1$Institute for Computational Biomedicine, RWTH Aachen University, Aachen, Germany. \\
    \small
    $^2$Center for Computational Life Sciences, RWTH Aachen University, Aachen, Germany. \\
    \small
    $^*$\textit{Correspondence:} aschuppert@ukaachen.de
}

\date{}

\begin{document}
\maketitle

\begin{abstract}
Hybrid mechanistic--data-driven models, which combine first-principles with learned components, are increasingly used in process engineering and scientific machine learning. Common hybrid modeling designs are specified primarily through their architectures and training losses, which offers a limited basis for a shared semantic interface to compare or verify them across domains, with comparatively little attention paid to epistemic uncertainty in the mechanistic part.

We bridge hybrid modeling and neuro-symbolic (NeSy) AI by reconstructing these designs as instances of NeSy interface. The resulting translation, Hybrid-to-NeSy (H2N), places mechanistic knowledge on the language side, learned modules on the belief side, and validity domains together with constraints on the logic side. For each design, H2N then yields an explicit NeSy inference functional and a logic--belief decomposition.

From this decomposition we derive two metrics: structural violation rate (SVR), measuring whether the learned belief respects the mechanistic structure; and belief dispersion (BD), measuring how concentrated the learned plausibility is, serving as a hybrid model's epistemic uncertainty in its mechanistic part.
We instantiate H2N on a case study of a structured hybrid model for binary classification under label noise and show that models with higher SVR and BD exhibit greater variability in held-out accuracy. Under structural distribution shift, H2N further quantifies a model's uncertainty during extrapolations, whereas test accuracy reveals the same shift only post hoc.

\end{abstract}

\section{Introduction}
\label{sec:intro}

In many science and engineering domains, the system of interest is partially understood: mass and energy balances, reaction stoichiometry, structural topology, or known kinetic forms specify what must be true, while the remaining components (constitutive closure, an unobserved rate, an environment-dependent residua) resist first-principles treatment and must be learned from data. \emph{Hybrid mechanistic--data-driven models}, also called semi-parametric \cite{teixeira2007hybrid} or grey-box \cite{tulleken1993grey} models, embed mechanistic knowledge while learning the rest \cite{psichogios1992hybrid,thompson1994modeling,van1996strategy,agarwal1997combining,schuppert2000extrapolability}. A canonical example: in a bioreactor model, mass balances and stoichiometry are fixed equations, while the cell-specific growth rate $\mu(c,T)$ is learned from data and inserted into the otherwise-mechanistic dynamics. The hybrid model thus inherits the interpretability \cite{cameron2001process}, data efficiency \cite{fiedler2008local}, and extrapolation properties \cite{van1998understanding} of mechanistic models, while using data-driven approximation where mechanism is incomplete.

Across application areas, hybrid modeling has converged on a handful of reusable design patterns \cite{von2014hybrid}: (i) \emph{serial} and \emph{parallel} arrangements of mechanistic and learned sub-models \cite{lee2002hybrid}, (ii) residual correction by output superposition \cite{bhutani2006first}, (iii) mixture-of-experts with gating \cite{peres2001knowledge,peres2008bioprocess} and fuzzy-rule analogues \cite{takagi1985fuzzy,van2002structured}, and (iv) validity-domain rules that modulate the learned component when inputs leave the training support \cite{kahrs2007validity,schuppert2011efficient,samadi2024hybrid}. 

Similar patterns appear throughout scientific machine learning \cite{glassey2018hybrid}, but they are described at the level of architectures and training losses. 
Two consequences follow. First, there is no widely adopted shared semantic account of what a hybrid model \textit{means} as an inference object, which can make cross-domain comparison and verification difficult. Second, hybrid models often lack a principled way to represent and propagate \textit{epistemic} uncertainty in mechanistic assumptions

Neuro-symbolic (NeSy) AI has long studied how to combine logical structure with learned components \cite{bader2005dimensions,hitzler2022neural,marra2024statistical,garcez2019neural,garcez2023neurosymbolic,van2021modular,belle2015probabilistic,marconato2023not}. De Smet and De Raedt \cite{de2025defining} recently consolidated this literature into a general definition: a NeSy model is a tuple $(L,\Sem,\OmegaSet,\Bel)$ where $L$ is a language with semantics $\Sem$ over an interpretation space $\OmegaSet$ and $\Bel$ is a belief weight, with inference defined by an integral functional
\begin{equation}
F_{\boldsymbol{\theta},x}(\varphi)
=
\int_{\Omega'} \LogicFn(\varphi,\omega)\, b_{\boldsymbol{\theta},x}(\omega)\,\mathrm{d}\Measure(\omega)
\label{eq:nesy-functional}
\end{equation}
that combines a logic function $\LogicFn$ (evaluating queries $\varphi$ in interpretations $\omega$) with the belief $b_{\boldsymbol{\theta},x}$. Crucially, this interface separates the \emph{logic side} (what is \textit{admissible}, given by $L$, $\Sem$, $\LogicFn$, and the integration domain $\Omega'\subseteq\OmegaSet$) from the \emph{belief side} (what is \textit{plausible}, given by $b_{\boldsymbol{\theta},x}$). %

Hybrid modeling and NeSy AI address closely related integration problems but from opposite directions: hybrid modeling has rich architectural patterns but limited basis for a shared semantic interface; NeSy has the semantics but emphasises logical languages rather than mechanistic equation systems and engineering patterns. Our work establishes a connection between these perspectives.

\paragraph{Thesis.}
Our central claim is that hybrid mechanistic--data-driven models can be
reconstructed, systematically, as NeSy models in the sense of De Smet and De Raedt
\cite{de2025defining}. Under this correspondence the mechanistic equations and
structural constraints supply the language $L$ and its semantics $\Sem$; the
learned components induce a belief function $\Bel$ over the unknown quantities;
and validity rules and constraints are expressed either as logic functions
$\LogicFn$ or as restrictions of the integration domain $\Omega'$. Every hybrid
architecture then induces an explicit inference functional of the
form~\eqref{eq:nesy-functional}.

\paragraph{Consequences of placement.}
Where each assumption is placed within the NeSy tuple $(L,\Sem,\OmegaSet,\Bel)$
determines which inference functionals are well-defined. In particular, for structured hybrid models of Ref. \cite{fiedler2008local}, keeping
the structural partition on the logic side (encoding it in the integration
domain $\Omega'$ rather than in the belief $\Bel$) yields an additive decomposition $\mathrm{BD}=\mathrm{BD}_{\text{seen}}+\mathrm{BD}_{\text{unseen}}$
(Section~\ref{sec:example}), in which $\mathrm{BD}_{\text{unseen}}$ is computable
from the coverage of $\Omega'$ at deployment time, before any out-of-distribution
(OOD) sample is observed. The conventional architecture-and-loss descriptions of
hybrid models express no such quantity, because they do not separate the
structural partition from the learned predictor.

\paragraph{Contributions.}
(i)~We introduce a principled translation \emph{procedure} that maps any hybrid model description into a NeSy tuple $(L,\mu,\Omega,b_{\boldsymbol{\theta}})$ with an
explicit inference functional (Sections~\ref{sec:problem}--\ref{sec:method}).
(ii)~Applying the procedure to canonical hybrid design patterns yields a compact \emph{mapping table} (Table~\ref{tab:taxonomy}) that records, for
each pattern, what changes in $(L,\mu,\Omega,b_{\boldsymbol{\theta}})$ and in the
induced functional, together with representative NeSy architectures that realize
it.
(iii)~We derive an evaluation protocol comprising two metrics that measure violations of logical structure and the learned belief's concentration (Section~\ref{sec:eval}).
(iv)~We instantiate the translation on a structured hybrid model for binary classification under label noise and show that the resulting metrics quantify the trained model’s epistemic uncertainty in the mechanistic component at deployment time, as well as uncertainty during extrapolation, failure modes that test accuracy alone does not reveal (Section~\ref{sec:example}).

\section{Problem Setting and Definitions}
\label{sec:problem}

\paragraph{Hybrid mechanistic--data-driven models.} Let $x\in\mathcal{X}$ denote inputs, $z\in\mathcal{Z}$ a latent state, and $y\in\mathcal{Y}$ an output.
Let $\alpha$ denote mechanistic parameters, and let $\psi$ collect unknown terms (residuals, gates, noise terms) required to complete the mechanistic description.
Here, a \emph{closure} is any constitutive specification of such unknown terms that renders the mechanistic constraints solvable once $x$ and $\alpha$ are fixed.
A hybrid model is specified by
\begin{align}
\textbf{(Mechanistic constraints)} 
&\quad 
\mathcal{G}(x,z;\alpha,\psi)=0,
\label{eq:hyb-mech}
\\
\textbf{(Observation model)} 
&\quad 
y=\mathcal{H}(x,z;\alpha,\psi)\in\mathcal{Y},
\label{eq:hyb-obs}
\\
\textbf{(Learned closure)} 
&\quad 
\psi = g_{\beta}\!\big(q(x,z),\xi\big),
\qquad \xi\sim p(\xi),
\label{eq:hyb-learned}
\end{align}
where $q$ is a feature map and $g_\beta$ is a learned module (deterministic or stochastic), and $\xi$ is a noise variable with a fixed base distribution $p(\xi)$.
Given $x$, the prediction $\hat y_{\alpha,\beta}(x)$ is obtained by solving \eqref{eq:hyb-mech}--\eqref{eq:hyb-learned} for $(z,y)$; if multiple solutions exist, the hybrid specification is assumed to include a fixed solver. 
If $g_\beta$ is stochastic, the specification induces a predictive distribution; a point prediction can be taken as $\hat y_{\alpha,\beta}(x):=\E[y\mid x]$.

\begin{definition}[Hybrid model]
A \emph{hybrid model} is the specification
$\mathcal{M}_{\alpha,\beta}=(\mathcal{G},\mathcal{H},g_\beta)$ together with the induced prediction map
$x\mapsto \hat y_{\alpha,\beta}(x)$.
\end{definition}

Many hybrid modeling designs can also be summarized at the architecture level as 
\begin{equation}
y = \operatorname{Comp}\!\left(M_{\alpha}(x),\, N_{\beta}\big(x,M_{\alpha}(x)\big)\right),
\label{eq:hyb-template}
\end{equation}
where $M_\alpha$ denotes the mechanistic operator, $N_\beta$ collects learned components, and 
$\operatorname{Comp}$ is the composition rule that combines the mechanistic output with the learned component.

\paragraph{NeSy models and the inference functional.} We adopt the interface of De Smet and De Raedt \cite{de2025defining}. A NeSy model specifies a language $L$, a semantics $\Sem$ over an interpretation space $\OmegaSet$, and a belief function $\Bel$ that weights interpretations; inference aggregates a logic function $\LogicFn$ against belief as in Equation~\eqref{eq:nesy-functional}. We distinguish the full space $\OmegaSet$ from the restricted space $\Omega'\subseteq\OmegaSet$ relevant to a given translation (unknown parameters, closures, latent states). Under mild measurability conditions on $(\OmegaSet,\LogicFn,B_{\boldsymbol{\theta},x})$, stated in Appendix~\ref{app:measure}, the functional~\eqref{eq:nesy-functional} is well-defined; Dirac beliefs correspond to $B_{\boldsymbol{\theta},x}=\delta_{\omega^\ast(x)}$.

\begin{definition}[Hybrid-to-NeSy decomposition]
\label{def:hyb-decomp}
A hybrid model $\mathcal{M}_{\alpha,\beta}$ induces a NeSy tuple $(L,\Sem,\OmegaSet,\Bel)$ and an induced inference functional $F$, where 
\begin{enumerate}
    \item $L$ and $\Sem$ capture the \emph{structural core} given by Equations \eqref{eq:hyb-mech}--\eqref{eq:hyb-obs} and the composition pattern in Equation \eqref{eq:hyb-template};
    \item $\OmegaSet$ is the \emph{interpretation space} of assignments to unknown quantities (e.g.\ $\alpha,\psi,z$, gates, noise);
    \item $\Bel$ is a \emph{belief function} over $\OmegaSet$ induced by learned part(s); (Dirac in deterministic settings; non-degenerate in stochastic/Bayesian settings);
    \item $F$ is the induced inference functional of the form Equation~\eqref{eq:nesy-functional}.
\end{enumerate}
\end{definition}

\begin{proposition}[Logic--belief separation]
\label{prop:separation}
Let $(L,\Sem,\OmegaSet,\Bel)$ be obtained via Definition~\ref{def:hyb-decomp}. Fix $x$, $\varphi\in L$, and measurable $\Omega'\subseteq\OmegaSet$.
Let $B_{\boldsymbol{\theta},x}$ be the induced (finite) belief measure on $\Omega'$ and assume $\LogicFn(\varphi,\cdot)$ is measurable and non-negative.
For $\tau\ge 0$ define
\[
\Omega_{\mathrm{viol}}^{\tau}(\varphi):=\{\omega\in\Omega' : \LogicFn(\varphi,\omega)\le \tau\},
\qquad
\Omega_{\mathrm{adm}}^{\tau}(\varphi):=\Omega'\setminus \Omega_{\mathrm{viol}}^{\tau}(\varphi).
\]
Then:
\begin{enumerate}
\item \emph{Logic side (admissibility).} The violation/admissibility partition $\Omega_{\mathrm{viol}}^{\tau}(\varphi)$ vs.\ $\Omega_{\mathrm{adm}}^{\tau}(\varphi)$ is determined entirely by $\LogicFn$ (hence by $(L,\Sem)$ and the chosen feasibility/scoring rule). In the hard case $\tau=0$,
$
F_{\boldsymbol{\theta},x}(\varphi)=\int_{\Omega'} \LogicFn(\varphi,\omega)\, \mathrm{d}B_{\boldsymbol{\theta},x}(\omega)
=\int_{\Omega_{\mathrm{adm}}^{0}(\varphi)} \LogicFn(\varphi,\omega)\, \mathrm{d}B_{\boldsymbol{\theta},x}(\omega).
$
\item \emph{Belief side (plausibility).} Conditional on admissibility, inference is governed by how $B_{\boldsymbol{\theta},x}$ allocates \emph{probability mass} within  $\Omega_{\mathrm{adm}}^{\tau}(\varphi)$ (concentrated vs.\ dispersed). In particular, if $B_{\boldsymbol{\theta},x}=\delta_{\omega^\ast(x)}$, then $F_{\boldsymbol{\theta},x}(\varphi)=\LogicFn(\varphi,\omega^\ast(x))$ and the induced behavior is admissible iff $\omega^\ast(x)\in\Omega_{\mathrm{adm}}^{\tau}(\varphi)$; non-degenerate beliefs distribute probability mass across multiple admissible interpretations.
\end{enumerate}
\end{proposition}

\begin{proof}[Proof sketch]
The admissibility partition uses only $\LogicFn(\varphi,\cdot)$ and $\tau$, hence only
$(L,\Sem,\LogicFn)$ and not $B_{\boldsymbol{\theta},x}$. In the hard case $\tau=0$,
$\LogicFn(\varphi,\cdot)=0$ on $\Omega_{\mathrm{viol}}^{0}(\varphi)$ by non-negativity
together with the definition of $\Omega_{\mathrm{viol}}^{0}(\varphi)$, so the violating
set contributes nothing; a Dirac belief collapses $F_{\boldsymbol{\theta},x}(\varphi)$ to
$\LogicFn(\varphi,\omega^\ast(x))$. The full argument is in Appendix~\ref{app:proof}.
\end{proof}

\begin{remark}[Coverage independent of belief]
\label{rem:coverage}
The partition $\Omega_{\mathrm{adm}}^\tau/\Omega_{\mathrm{viol}}^\tau$ depends
only on $(L,\Sem,\LogicFn)$, so structural coverage of $\Omega'$ is a function
of the logic side alone. Sec.~\ref{sec:example} uses this independence to
compute $\mathrm{BD}_{\text{unseen}}$ from row coverage, without access to the
OOD sample.
\end{remark}

\noindent\textbf{Problem.}
Given a hybrid model (Equation~\eqref{eq:hyb-template}), construct $(L,\Sem,\OmegaSet,\Bel)$ and $\LogicFn$ such that structures are evaluated via $\LogicFn$ under $(L,\Sem)$, unknowns live in $\OmegaSet$, learning induces $\Bel$, and the induced functional $F$ reproduces the hybrid's predictive behaviour (up to selection rules).

\section{Method: A translation from hybrid modeling to NeSy}
\label{sec:method}

This section gives a constructive solution to \textsc{H2N}: given a hybrid design, we build $(L,\Sem,\OmegaSet,\Bel)$ and $\LogicFn$ so that the induced functional in~\eqref{eq:nesy-functional} becomes an explicit, comparable interface. The input $x$ is external to $\omega$; the learned part induces $b_{\boldsymbol{\theta},x}(\omega)$ conditional on $x$, while query dependence enters through $\LogicFn(\varphi,\omega)$. If a design requires query-dependent beliefs, this can be represented by extending $\omega$ or letting $b_{\boldsymbol{\theta},x}$ depend on $\varphi$; our default conditions on $x$ only.

\subsection{Translation principles}
\label{subsec:principles}

\textsc{H2N} is governed by five principles, stated in full in
Appendix~\ref{app:principles} and summarised here:\\
\textbf{(P1)}~Mechanistic structure determines the language $L$, with unknown
closures and latent states introduced as explicit symbols.\\
\textbf{(P2)}~Semantics $\Sem$ specify evaluation; the logic function $\LogicFn$
is the operational object inside the integral, implementing admissibility or
scoring.\\
\textbf{(P3)}~Learned modules induce belief
$b_{\boldsymbol{\theta},x}(\omega)$ over interpretations; deterministic
learners are the Dirac case.\\
\textbf{(P4)}~Validity domains and hard constraints live on the logic side:
by restricting the integration domain $\Omega'\subseteq\OmegaSet$ and/or by
selection inside $\LogicFn$, separating \emph{admissibility} (logic) from
\emph{plausibility} (belief). 
Hard physical laws are encoded on the logic side rather than in $\Bel$: such
laws express data-independent admissibility, whereas $\Bel$ is by construction
data-dependent. Encoding them in $\Bel$ would couple their enforcement to
sample size, which is inconsistent with their meaning. Soft or empirical
constraints, whose strength legitimately depends on data, may enter $\Bel$.\\ 
\textbf{(P5)}~Architectural factorisations (serial, parallel, mixture) are
mirrored by factorisations of $\LogicFn$ and/or $b_{\boldsymbol{\theta},x}$.
As summarised in Algorithm~\ref{alg:h2n} in
Appendix~\ref{app:principles}, the result is a NeSy inference
object whose structural core $(L,\Sem,\LogicFn)$ and uncertainty model
$\Bel$ make comparison of hybrid designs systematic.

\subsection{Mapping of hybrid modeling designs to NeSy elements}
\label{subsec:taxonomy}

Table~\ref{tab:taxonomy} summarizes how canonical hybrid modeling design patterns map to NeSy elements and how each pattern changes the inference functional $F_{\boldsymbol{\theta},x}$ in Equation~\eqref{eq:nesy-functional}. Across the patterns, the translation separates structure from learned uncertainty. 
The main differences are whether a pattern changes (i) the structural language $L$ that determines what is expressible, (ii) operational admissibility or scoring through $\LogicFn$ that determines how structure is enforced, (iii) which unknowns are integrated over through the choice of $\Omega'$ within $\OmegaSet$, or (iv) how uncertainty is allocated through belief factorization and dispersion in $b_{\boldsymbol{\theta},x}$.

\begin{sidewaystable}[!p]
\centering
\footnotesize
\setlength{\tabcolsep}{3pt}
\renewcommand{\arraystretch}{1.15}
\caption{Canonical hybrid modeling design patterns mapped to NeSy elements and their primary impact on the inference functional $F_{\boldsymbol{\theta},x}$ in Equation~\eqref{eq:nesy-functional}. The final column lists representative NeSy architectures that realise each pattern.}
\label{tab:taxonomy}

\begin{tabular}{
>{\raggedright\arraybackslash}p{0.20\textheight}
>{\raggedright\arraybackslash}p{0.11\textheight}
>{\raggedright\arraybackslash}p{0.11\textheight}
>{\raggedright\arraybackslash}p{0.11\textheight}
>{\raggedright\arraybackslash}p{0.11\textheight}
>{\raggedright\arraybackslash}p{0.11\textheight}
>{\raggedright\arraybackslash}p{0.13\textheight}
}
\toprule
\textbf{Hybrid design} &
\textbf{$L$} &
\textbf{$\Sem$ / $\LogicFn$} &
\textbf{$\OmegaSet$ (and $\Omega'$)} &
\textbf{$\Bel$ (via $b_{\boldsymbol{\theta},x}$)} &
\textbf{Effect on $F_{\boldsymbol{\theta},x}$} &
\textbf{Example NeSy realisations} \\
\midrule

\textbf{Serial closure} (learned kinetics/transport inside balances) \cite{psichogios1992hybrid,thompson1994modeling} &
Balance equations with explicit closure symbols &
Residual or trajectory evaluation; feasibility selection &
Closures, parameters, latent states, noise; $\Omega'$ restricts to relevant unknowns &
Belief over closures/params (Dirac or probabilistic) &
$\LogicFn$ enforces feasibility; belief integrates closure uncertainty &
DeepProbLog-style amortised inference; neuro-symbolic concept learners \\

\addlinespace
\textbf{Parallel residual} (mechanistic + learned correction) \cite{lee2002hybrid,bhutani2006first} &
Constraint tying $y$ to $y_{\mathrm{mech}}$ and residual symbol &
Often graded scoring; constraints as penalty &
Residual map, bias/noise terms; $\Omega'$ may implement trust region &
Belief over residual/noise; optional drift priors &
Architectural correction becomes either $\Omega'$ restriction or $\LogicFn$ penalty &
Logic-constrained neural networks; residual-corrected neural ODEs (sci-ML) \\

\addlinespace
\textbf{Mixture-of-experts} (with gating or weighting) \cite{peres2001knowledge,peres2008bioprocess} &
Expert submodels plus gate symbols &
Hard selection (Boolean) or soft mixing (graded) &
Gate variables, mixture weights, expert latents; $\Omega'$ can enforce regime constraints &
Belief over gate/weights and experts &
Hard/soft regimes shift expressivity between $\LogicFn$ and $b_{\boldsymbol{\theta},x}$ &
LTN with gating; attention-mixture symbolic learners \\

\addlinespace
\textbf{Rule + physics} (fuzzy rules with constraints) \cite{van2002structured,van2003combining} &
Rule predicates plus mechanistic formulas &
Fuzzy truth + constraint penalties/feasibility &
Memberships, rule weights, constraint slack; $\Omega'$ may encode admissible slack &
Belief over memberships/weights &
Constraints handled by $\Omega'$ restriction and/or $\LogicFn$ penalization &
Logic Tensor Networks; Semantic Loss; KENN \\

\addlinespace
\textbf{Modular structures} (Boolean module graphs) \cite{fiedler2008local,e2022training} &
Wiring constraints and module predicates &
Boolean satisfaction; robust variants via tolerance in $\LogicFn$ &
Module tables/outputs, latent flips; $\Omega'$ restricts to queried symbols &
Belief over module uncertainty (Dirac or distribution) &
Robustness expressed by tolerance in $\LogicFn$ vs dispersion in belief &
Modular Boolean networks \\

\bottomrule
\end{tabular}
\end{sidewaystable}

\section{Evaluation Protocol}
\label{sec:eval}

Let $\mathcal{D}_{\mathrm{test}}=\{(x_i,y_i)\}_{i=1}^n$ be a test set, and let $F_{\boldsymbol{\theta},x}$ be the induced inference functional in Equation~\eqref{eq:nesy-functional}. We evaluate the H2N translation with two kinds of queries, each evaluable through $\LogicFn(\varphi,\omega)$ and integrable under $b_{\boldsymbol{\theta},x}$. A \textbf{constraint query} $\varphi^{\mathrm{con}}_{x}$ expresses structural admissibility; a \textbf{predictive query} $\varphi^{\mathrm{pred}}_{x,y}$ expresses agreement with data (e.g.\ $\|\hat y(x,\omega)-y\|\le\varepsilon$ for a tolerance $\varepsilon$), of which ordinary test accuracy is the point-estimate reading. The protocol is otherwise agnostic to the domain.

For a constraint query $\varphi^{\mathrm{con}}_{x}$ and a tolerance $\tau\ge 0$, the logic--belief separation partitions $\Omega'$ into an admissible region $\Omega_{\mathrm{adm}}^{\tau}(\varphi)$ and a violating region $\Omega_{\mathrm{viol}}^{\tau}(\varphi)$. This partition is fixed entirely by $\LogicFn$ (hence by $(L,\Sem)$ and the chosen feasibility rule), independently of the belief (Proposition~\ref{prop:separation}, Remark~\ref{rem:coverage}). The belief-weighted violation mass at $x$,
\[
V_{\boldsymbol{\theta}}(x;\varphi) := \int_{\Omega'} \Ind{\LogicFn(\varphi,\omega)\le \tau}\, \mathrm{d}B_{\boldsymbol{\theta},x}(\omega),
\qquad \mathrm{d}B_{\boldsymbol{\theta},x}=b_{\boldsymbol{\theta},x}\,\mathrm{d}\Measure,
\]
is the (normalised) belief mass falling in the violating region; it couples the logic-side partition to the belief and reduces to $\int_{\Omega'}\Ind{\LogicFn(\varphi,\omega)=0}\,\mathrm{d}B_{\boldsymbol{\theta},x}$ in the hard Boolean case $\tau=0$.

\subsection{Metrics}
\label{subsec:eval-metrics}

\paragraph{Structural violation rate (SVR).} Given a constraint query $\varphi^{\mathrm{con}}_{x}$,
\begin{equation}
\mathrm{SVR}(\boldsymbol{\theta}) := \frac{1}{n}\sum_{i=1}^n V_{\boldsymbol{\theta}}(x_i;\varphi^{\mathrm{con}}_{x_i}) ,
\end{equation}
the average belief mass placed on interpretations that violate the structural constraints. SVR is near $0$ when the logic side reliably admits interpretations that satisfy the constraints under the learned belief, and increases as belief mass concentrates on invalid interpretations. Three properties characterise SVR. \emph{(i)~It lives on the logic side:} its value is governed by the admissibility partition $\Omega_{\mathrm{adm}}^{\tau}/\Omega_{\mathrm{viol}}^{\tau}$, a property of $(L,\Sem,\LogicFn)$, weighted by the belief. \emph{(ii)~It is tolerance-dependent:} the strictness of the constraint is set by $\tau$ (the violation budget $\kappa$ in the case study of Section~\ref{sec:example}); tightening $\tau$ enlarges the violating region and raises SVR, loosening it lowers SVR.  \emph{(iii)~It is a feasibility measure, not a dispersion measure:} for a Dirac belief, $V_{\boldsymbol{\theta}}(x;\cdot)\in\{0,1\}$ records only whether the point estimate is admissible, and SVR says nothing about how concentrated the belief is.

\paragraph{Belief dispersion (BD).}
To assess epistemic uncertainty, we choose an uncertain projection $\omega_U=\pi_U(\omega)$ (e.g.\ closure values, parameters) and report
\begin{equation}
\mathrm{BD}(\boldsymbol{\theta}) := \frac{1}{n}\sum_{i=1}^n \mathrm{tr}\!\left(\mathrm{Cov}_{\omega\sim b_{\boldsymbol{\theta},x_i}}\![\omega_U]\right),
\end{equation}
where sampling $\omega\sim b_{\boldsymbol{\theta},x}$ is from the normalised belief on $\Omega'$ (proportional to $b_{\boldsymbol{\theta},x}(\omega)\,d\Measure(\omega)$). BD has the dual set of properties. \emph{(i)~It lives on the belief side:} it is a functional of the second moment of $b_{\boldsymbol{\theta},x}$ alone, and the logic function $\LogicFn$ does not enter. \emph{(ii)~It is tolerance-independent:} BD never references $\tau$ or the admissibility partition, so changing the constraint's strictness leaves it unchanged. \emph{(iii)~It is a dispersion measure:} BD is zero for Dirac beliefs and grows as the belief spreads; it is invariant under orthonormal reparameterisations of $\omega_U$ but not scale-invariant, so $\omega_U$ should be standardised when coordinates differ in scale.

\section{Case Study: Structured networks for noisy binary classification}
\label{sec:example}

We instantiate \textsc{H2N} on structured Boolean networks for binary classification with prior feature grouping. The mechanistic part is a partition of the $N$ binary-represented features into $M$ disjoint groups routed to first-layer Boolean modules $F_m:\{0,1\}^{n_m}\to\{0,1\}$, combined by an output module $F_O:\{0,1\}^{M}\to\{0,1\}$. The learned components are the module truth-tables, identified by the learning strategy of NoiseCut~\cite{samadi2024noisecut}.

\paragraph{\textsc{H2N} translation.}
The unknowns are the module truth tables, so we introduce one Boolean atom per table entry: $f_{m,k}$ encodes $F_m(k)$ at local input $k\in\{0,1\}^{n_m}$, and $o_K$ encodes $F_O(K)$ at second-layer \textit{row} $K\in\{0,1\}^M$. 
An interpretation $\omega\in\OmegaSet=\Omega'=\B^{\atoms}$ ($\atoms$ the set of all atoms) is then a complete assignment of module outputs, and induces a predictor: 
routing $x$ through the first layer selects the row $K^{\omega}(x)=\big(\omega(f_{1,k_1(x)}),\dots,\omega(f_{M,k_M(x)})\big)$, with predicted label $O^{\omega}(x)=\omega(o_{K^{\omega}(x)})$. 
Under Boolean semantics $\Sem=\mu_B$, per-sample consistency is $\varphi_s=\big(O^{\omega}(x^{(s)})\!\leftrightarrow\! y^{(s)}\big)$ and the constraint query is $\varphi^{\mathrm{con}}=\bigwedge_{s=1}^{S}\varphi_s$ over the $S$ training samples. 
Writing $\mathrm{viol}(\omega)=\sum_s\Ind{\mu_B(\varphi_s,\omega)=0}$
for the number of training samples misclassified by the interpretation $\omega$, the noise robustness of the logical side is determined by the choice of logical evaluation function:
\[
\LogicFn_{\mathrm{hard}}=\mu_B(\varphi^{\mathrm{con}},\omega),\qquad
\LogicFn_{\mathrm{tol}}=\Ind{\mathrm{viol}(\omega)\le\kappa}.
\]
These correspond respectively to exact consistency ($\kappa=0$) and consistency subject to a violation budget $\kappa$. Consequently, noise robustness is a design choice of the logical evaluation function rather than a property of the learned belief; SVR is evaluated under $\LogicFn_{\mathrm{tol}}$, with the budget $\kappa$ playing the role of the tolerance $\tau$ of Section~\ref{sec:eval}.

\paragraph{Belief.}
NoiseCut identifies the truth tables of the first-layer modules, yielding point estimates $\hat F_m$, which are incorporated as Dirac factors. The truth table of the output module is estimated via label counting: for each row $K$, let $n_0(K)$ and $n_1(K)$ denote the numbers of training samples labeled $0$ and $1$, respectively. This defines a Bernoulli parameter $p_K = n_1(K)/(n_0(K)+n_1(K))$
with $p_K=\tfrac12$ assigned to rows receiving no training support (i.e., $n_0(K)+n_1(K)=0$). 
The belief factorises as
\[
b_{\boldsymbol{\theta}}(\omega)=
\underbrace{\prod_m \Ind{\omega|_{F_m}=\hat F_m}}_{\text{Dirac (first-layer modules)}}\;
\underbrace{\prod_K p_K^{\omega(o_K)}(1-p_K)^{1-\omega(o_K)}}_{\text{Bernoulli (output module)}}.
\]
The non-degenerate Bernoulli factors capture the epistemic uncertainty, while unobserved rows are assigned the maximum-entropy value $p_K=\tfrac12$. Under the counting measure, the functional in~\eqref{eq:nesy-functional} corresponds to a weighted model count. The corresponding estimators are described in Appendix~\ref{app:mc-protocol}.

\paragraph{Experimental setup.}
The data are generated from random functionality-preserving modules with $n_m=4$
inputs per first-layer module. All experiments use $M=4$ first-layer
modules ($N=16$, $2^{16}$ samples) and a $50\%/50\%$ train/test split. The noise sweep study uses a fixed tolerance $\kappa=\lfloor 0.5\,S\rfloor$, and $50$ seeds at each of ten noise levels from $0\%$ to $45\%$. The
tolerance sweep uses $10$ seeds and fixes $10\%$ label noise and sweeps
$\kappa/S$ from $0$ to $0.5$ (Appendix~\ref{app:tolerance-sweep}); the OOD study uses $10$ seeds and 
fixes $\kappa=\lfloor 0.10\,S\rfloor$ and $10\%$ label noise and holds out
$N_{\mathrm{ood}}\in\{1,2,4,6\}$ of the $16$ reachable second-layer rows (Appendix~\ref{app:ood-extra}). Label noise flips training labels only, keeping testing labels clean. All H2N metrics use $20{,}000$ Monte~Carlo $F_O$ draws,
with BD additionally available in closed form (Appendix~\ref{app:mc-protocol}).

\paragraph{Result~1: quantifying uncertainty in the mechanistic part.}
Table~\ref{tab:noise-sweep} and Figure~\ref{fig:noise-sweep} report the H2N
metrics across label-noise levels. As the noise increases, mean test accuracy
falls from $1.00$ to $0.63$ while mean BD measured at the deployment time rises from $0$ to near its maximum
$2^{M}/4=4$, the value attained when every second-layer row is maximally
uncertain, $p_K=\tfrac12$.  
Test accuracy and BD are strongly anti-correlated across seeds (Spearman $\rho=-0.94$), 
and therefore track related but distinct properties: the former is a point estimate
of predictive error, the latter the dispersion of the row-Bernoulli belief.

As shown in Table~\ref{tab:noise-sweep} (as well as in Figure~\ref{fig:noise-appendix}), SVR, under the tolerance $\kappa=\lfloor 0.5\,S\rfloor$, remains near zero while the realized violation count stays within the budget and rises to $0.428$ only at at $45\%$ noise levels, where violations routinely exceed $\kappa$. 

\begin{figure}[h]
\centering
\includegraphics[width=0.99\linewidth]{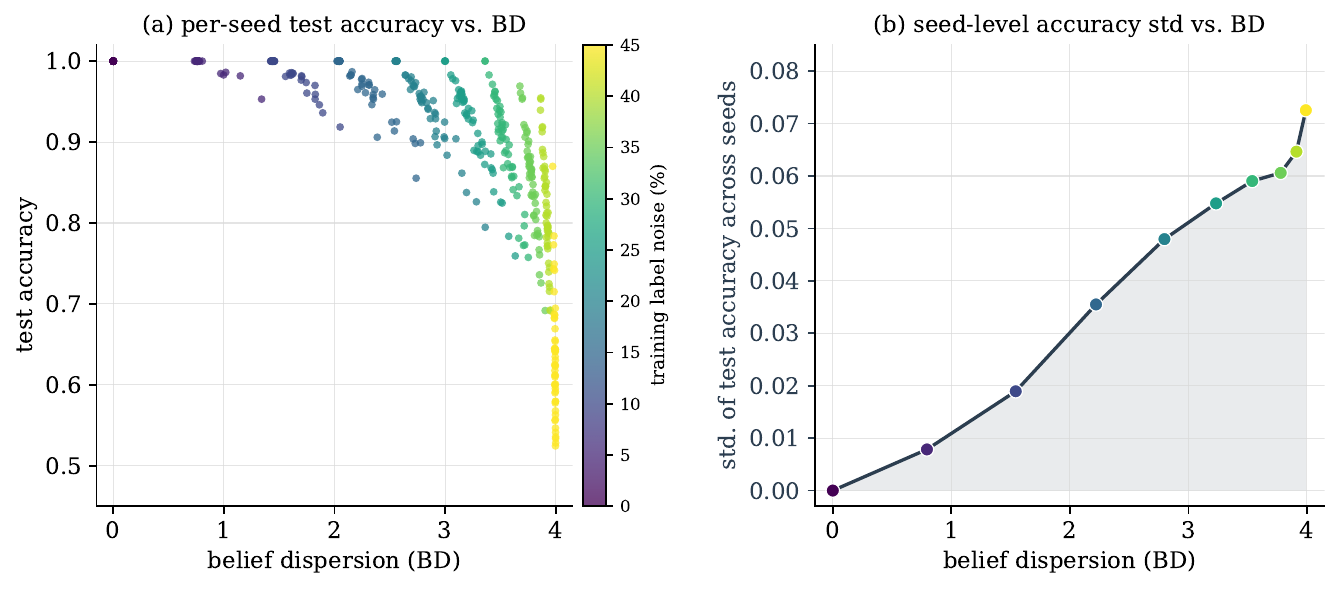}
\caption{Noise sweep, $50$ seeds per level (struct $[4,4,4,4]$, $\kappa=\lfloor 0.5\,S\rfloor$). (a) Per-seed scatter of test accuracy against BD, coloured by noise level: the spread of accuracy widens as BD grows. (b) Standard deviation of test accuracy across seeds, computed within bins of BD; 
the accuracy uncertainty rises monotonically with BD.
}
\label{fig:noise-sweep}
\end{figure}

\begin{table}[h]
\centering
\caption{%
  Noise sweep with H2N metrics; struct $[4,4,4,4]$, train fraction $50\%$,
  $\kappa=\lfloor0.5\,S\rfloor$.
  Each entry reports the mean over $50$ seeds with the
  $95\%$ confidence interval. 
  SVR and BD are measured at the deployment time using training data with label noise. Acc. and F1 score are measured using test data with clean labels.
}
\label{tab:noise-sweep}
\small
\setlength{\tabcolsep}{5pt}
\begin{tabular}{l
                S[table-format=1.3, table-space-text-post={(0.000)}]
                S[table-format=1.3]
                S[table-format=1.3, table-space-text-post={(0.000)}]
                S[table-format=1.3, table-space-text-post={(0.000)}]
                S[table-format=1.3, table-space-text-post={(0.000)}]}
\toprule
  {Noise}
  & {Acc.\ ($\pm$\,95\% CI)}
  & {Acc.\ std}
  & {F1 ($\pm$\,95\% CI)}
  & {SVR ($\pm$\,95\% CI)}
  & {BD ($\pm$\,95\% CI)} \\
\midrule
$0\%$  & 1.000 {(0.000)} & 0.000 & 1.000 {(0.000)} & 0.000 {(0.000)} & 0.000 {(0.000)} \\
$5\%$  & 0.998 {(0.002)} & 0.008 & 0.997 {(0.003)} & 0.000 {(0.000)} & 0.794 {(0.032)} \\
$10\%$ & 0.989 {(0.005)} & 0.019 & 0.989 {(0.005)} & 0.000 {(0.000)} & 1.544 {(0.045)} \\
$15\%$ & 0.973 {(0.010)} & 0.035 & 0.970 {(0.012)} & 0.002 {(0.001)} & 2.221 {(0.062)} \\
$20\%$ & 0.948 {(0.014)} & 0.048 & 0.944 {(0.017)} & 0.011 {(0.002)} & 2.798 {(0.057)} \\
$25\%$ & 0.922 {(0.016)} & 0.055 & 0.916 {(0.019)} & 0.036 {(0.005)} & 3.234 {(0.043)} \\
$30\%$ & 0.898 {(0.017)} & 0.059 & 0.890 {(0.022)} & 0.083 {(0.007)} & 3.539 {(0.026)} \\
$35\%$ & 0.851 {(0.017)} & 0.061 & 0.839 {(0.025)} & 0.170 {(0.008)} & 3.779 {(0.014)} \\
$40\%$ & 0.819 {(0.018)} & 0.065 & 0.803 {(0.028)} & 0.275 {(0.007)} & 3.912 {(0.007)} \\
$45\%$ & 0.633 {(0.021)} & 0.073 & 0.545 {(0.066)} & 0.428 {(0.006)} & 3.991 {(0.002)} \\
\bottomrule
\end{tabular}
\end{table}

The central observation concerns not the means but the \emph{spread} of accuracy.
When the $500$ trained models are stratified by their BD value, the seed-to-seed
standard deviation of test accuracy increases monotonically with BD, from $0$ at
$\mathrm{BD}\approx 0$ to about $0.073$ as $\mathrm{BD}\to 4$
(Figure~\ref{fig:noise-sweep}(b)); the per-seed scatter in
Figure~\ref{fig:noise-sweep}(a) exhibits the same effect. A high BD thus signals a wider, less predictable accuracy distribution. Because BD is a
function only of the training-label counts $n_0(K),n_1(K)$ and the mechanistic structure of 
first-layer partitions, it provides an \textit{a priori} estimate of how uncertain a trained model's accuracy will be, a quantity that test accuracy
can report only after labeled test data have been collected.

\paragraph{Result~2: quantifying uncertainty during extrapolation.}
We hold $N_{\mathrm{ood}}\in\{1,2,4,6\}$ reachable second-layer rows entirely out of
training (with $10\%$ noise) and stratify the test set into
covered and held-out rows, for which predictions amounts to extrapolation. Because the row partition is logical (fixed by
$\Omega'$ rather than by the learned predictor) BD decomposes exactly as
$\mathrm{BD}=\mathrm{BD}_{\text{seen}}+\mathrm{BD}_{\text{unseen}}$, where
$\mathrm{BD}_{\text{unseen}}=\tfrac14\,|\{K:n_0(K)+n_1(K)=0\}|$ sums the
maximum-entropy variance of the unobserved rows and is read off the \emph{learned}
first-layer functions. Crucially, $\mathrm{BD}_{\text{unseen}}$ is computable from
row coverage alone, before any OOD sample is seen
(Remark~\ref{rem:coverage}). Across $10$ seeds, mean in-distribution accuracy
stays at $0.99$--$1.00$ while mean OOD or extrapolation accuracy collapses to $0.35, 0.35, 0.56,
0.49$ and mean $\mathrm{BD}_{\text{unseen}}$ rises monotonically as
$0.25, 0.50, 0.68, 0.95$; the full numbers are in
Appendix~\ref{app:ood-extra}. The decomposition thus reports the model's epistemic
state about coverage from $\Omega'$ directly, whereas accuracy reveals the same
shift only after the OOD labels are observed.

\section{Discussion}
\label{sec:discussion}

Hybrid models are often presented as architectures and losses. \textsc{H2N} reconstructs each design as a NeSy inference object $(L,\Sem,\Omega,\Bel)$ with an explicit functional $F$, making it clear which assumptions act as \emph{structural admissibility} (language/semantics/logic function) versus \emph{learned plausibility} (belief over unknown parameters, functions, or states). This does not remove modeling choices (hard vs.\ soft constraints, choice of $\Sem$, belief parameterisation), but makes them comparable across canonical hybrid design patterns. SVR quantifies how much belief mass violates the structural constraints (logic side), while BD summarises epistemic dispersion of the learned belief (belief side); the two are complementary and decoupled, since the tolerance controls SVR without affecting BD. Robustness under structured shift becomes a targeted intervention: modify $\LogicFn$ (changed regimes/constraints) versus modify $\Bel$.

\paragraph{Measurable consequence of placement.}
Sec.~\ref{sec:example} instantiates the placement argument empirically:
BD quantifies the model's uncertainty on its structure and on uncovered regions of $\Omega'$ at deployment time, while test accuracy reports the same shift only post hoc. The diagnostic follows from logic--belief separation (Prop.~\ref{prop:separation}, Rem.~\ref{rem:coverage}) and is well-defined under H2N but undefined under the architecture-and-loss
description of the same model.

\paragraph{What NeSy gains from hybrid modeling.}
Hybrid modeling contributes assets that NeSy has lacked an interface to: an established library of \emph{compositional patterns} (serial closure, parallel residual, mixture-of-experts, modular networks) with documented identifiability and validity-domain analyses; \emph{principled constraint families} (conservation, balance, monotonicity) that translate naturally into $\LogicFn$ rather than into ad-hoc penalty terms; and \emph{diagnostic practices} from process engineering (validity envelopes, residual checks) that the H2N metrics generalise. The flow is therefore bidirectional: \textsc{H2N} gives hybrid modelers a semantic interface, and gives NeSy researchers a pre-processed catalogue of structured-uncertainty designs to import.

\paragraph{Limitations.}
The translation is not unique: redistributing assumptions between $\Omega'$ and
$\LogicFn$ can alter $F$ unless normalisation is controlled. Principles P1--P5
fix a canonical placement per pattern (Table~\ref{tab:taxonomy}); within this
discipline, alternative placements correspond to alternative modeling
commitments rather than equivalent notations. For large or continuous
$\OmegaSet$, diagnostics rely on approximate inference and inherit its
calibration error.

\bibliographystyle{unsrt}    
\bibliography{nesy2026-sample}

@article{de2025defining,
  title={Defining neurosymbolic ai},
  author={De Smet, Lennert and De Raedt, Luc},
  journal={arXiv preprint arXiv:2507.11127},
  year={2025}
}

@book{cameron2001process,
  title={Process modelling and model analysis},
  author={Cameron, Ian T and Hangos, Katalin},
  volume={4},
  year={2001},
  publisher={Elsevier}
}

@article{fiedler2008local,
  title={Local identification of scalar hybrid models with tree structure},
  author={Fiedler, Bernold and Schuppert, Andreas},
  journal={IMA Journal of Applied Mathematics},
  volume={73},
  number={3},
  pages={449--476},
  year={2008},
  publisher={OUP}
}

@article{thompson1994modeling,
  title={Modeling chemical processes using prior knowledge and neural networks},
  author={Thompson, Michael L and Kramer, Mark A},
  journal={AIChE Journal},
  volume={40},
  number={8},
  pages={1328--1340},
  year={1994},
  publisher={Wiley Online Library}
}

@article{van1996strategy,
  title={Strategy for dynamic process modeling based on neural networks in macroscopic balances},
  author={Van Can, Henricus JL and Hellinga, Chris and Luyben, Karel Ch AM and Heijnen, Joseph J and Te Braake, Hubert AB},
  journal={AIChE Journal},
  volume={42},
  number={12},
  pages={3403--3418},
  year={1996},
  publisher={Wiley Online Library}
}

@incollection{schuppert2000extrapolability,
  title={Extrapolability of structured hybrid models: a key to optimization of complex processes},
  author={Schuppert, Andreas A},
  booktitle={Equadiff 99: (In 2 Volumes)},
  pages={1135--1151},
  year={2000},
  publisher={World Scientific}
}

@article{psichogios1992hybrid,
  title={A hybrid neural network-first principles approach to process modeling},
  author={Psichogios, Dimitris C and Ungar, Lyle H},
  journal={AIChE Journal},
  volume={38},
  number={10},
  pages={1499--1511},
  year={1992},
  publisher={Wiley Online Library}
}

@article{van1998understanding,
  title={Understanding and applying the extrapolation properties of serial gray-box models},
  author={Van Can, Henricus JL and Te Braake, Hubert AB and Dubbelman, Sander and Hellinga, Chris and Luyben, Karel Ch AM and Heijnen, Joseph J},
  journal={AIChE journal},
  volume={44},
  number={5},
  pages={1071--1089},
  year={1998},
  publisher={Wiley Online Library}
}

@article{von2014hybrid,
  title={Hybrid semi-parametric modeling in process systems engineering: Past, present and future},
  author={Von Stosch, Moritz and Oliveira, Rui and Peres, Joana and De Azevedo, Sebasti{\~a}o Feyo},
  journal={Computers \& Chemical Engineering},
  volume={60},
  pages={86--101},
  year={2014},
  publisher={Elsevier}
}

@article{kahrs2007validity,
  title={The validity domain of hybrid models and its application in process optimization},
  author={Kahrs, Olaf and Marquardt, Wolfgang},
  journal={Chemical Engineering and Processing: Process Intensification},
  volume={46},
  number={11},
  pages={1054--1066},
  year={2007},
  publisher={Elsevier}
}

@article{peres2008bioprocess,
  title={Bioprocess hybrid parametric/nonparametric modelling based on the concept of mixture of experts},
  author={Peres, J and Oliveira, R and de Azevedo, S Feyo},
  journal={Biochemical Engineering Journal},
  volume={39},
  number={1},
  pages={190--206},
  year={2008},
  publisher={Elsevier}
}

@article{peres2001knowledge,
  title={Knowledge based modular networks for process modelling and control},
  author={Peres, J and Oliveira, R and De Azevedo, S Feyo},
  journal={Computers \& Chemical Engineering},
  volume={25},
  number={4-6},
  pages={783--791},
  year={2001},
  publisher={Elsevier}
}

@article{bhutani2006first,
  title={First-principles, data-based, and hybrid modeling and optimization of an industrial hydrocracking unit},
  author={Bhutani, N and Rangaiah, GP and Ray, AK},
  journal={Industrial \& engineering chemistry research},
  volume={45},
  number={23},
  pages={7807--7816},
  year={2006},
  publisher={ACS Publications}
}

@article{lee2002hybrid,
  title={Hybrid neural network modeling of a full-scale industrial wastewater treatment process},
  author={Lee, Dae Sung and Jeon, Che Ok and Park, Jong Moon and Chang, Kun Soo},
  journal={Biotechnology and bioengineering},
  volume={78},
  number={6},
  pages={670--682},
  year={2002},
  publisher={Wiley Online Library}
}

@article{takagi1985fuzzy,
  title={Fuzzy identification of systems and its applications to modeling and control},
  author={Takagi, Tomohiro and Sugeno, Michio},
  journal={IEEE transactions on systems, man, and cybernetics},
  number={1},
  pages={116--132},
  year={1985},
  publisher={IEEE}
}

@article{van2002structured,
  title={A structured modeling approach for dynamic hybrid fuzzy-first principles models},
  author={van Lith, Pascal F and Betlem, Ben HL and Roffel, Brian},
  journal={Journal of Process Control},
  volume={12},
  number={5},
  pages={605--615},
  year={2002},
  publisher={Elsevier}
}

@article{van2003combining,
  title={Combining prior knowledge with data driven modeling of a batch distillation column including start-up},
  author={van Lith, Pascal F and Betlem, Ben HL and Roffel, Brian},
  journal={Computers \& chemical engineering},
  volume={27},
  number={7},
  pages={1021--1030},
  year={2003},
  publisher={Elsevier}
}

@article{schuppert2011efficient,
  title={Efficient reengineering of meso-scale topologies for functional networks in biomedical applications},
  author={Schuppert, Andreas A},
  journal={Journal of Mathematics in Industry},
  volume={1},
  number={1},
  pages={6},
  year={2011},
  publisher={Springer}
}

@article{e2022training,
  title={A training strategy for hybrid models to break the curse of dimensionality},
  author={Samadi, Moein E and Kiefer, Sandra and Fritsch, Sebastian Johaness and Bickenbach, Johannes and Schuppert, Andreas},
  journal={Plos one},
  volume={17},
  number={9},
  pages={e0274569},
  year={2022},
  publisher={Public Library of Science San Francisco, CA USA}
}

@article{samadi2024hybrid,
  title={A hybrid modeling framework for generalizable and interpretable predictions of ICU mortality across multiple hospitals},
  author={Samadi, Moein E and Guzman-Maldonado, Jorge and Nikulina, Kateryna and Mirzaieazar, Hedieh and Sharafutdinov, Konstantin and Fritsch, Sebastian Johannes and Schuppert, Andreas},
  journal={Scientific reports},
  volume={14},
  number={1},
  pages={5725},
  year={2024},
  publisher={Nature Publishing Group UK London}
}

@article{samadi2024noisecut,
  title={Noisecut: a python package for noise-tolerant classification of binary data using prior knowledge integration and max-cut solutions},
  author={Samadi, Moein E and Mirzaieazar, Hedieh and Mitsos, Alexander and Schuppert, Andreas},
  journal={BMC bioinformatics},
  volume={25},
  number={1},
  pages={155},
  year={2024},
  publisher={Springer}
}

@article{agarwal1997combining,
  title={Combining neural and conventional paradigms for modelling, prediction and control},
  author={Agarwal, Mukul},
  journal={International Journal of Systems Science},
  volume={28},
  number={1},
  pages={65--81},
  year={1997},
  publisher={Taylor \& Francis}
}

@book{glassey2018hybrid,
  title={Hybrid modeling in process industries},
  author={Glassey, Jarka and Von Stosch, Moritz},
  year={2018},
  publisher={CRC Press}
}

@article{tulleken1993grey,
  title={Grey-box modelling and identification using physical knowledge and Bayesian techniques},
  author={Tulleken, Herbert JAF},
  journal={Automatica},
  volume={29},
  number={2},
  pages={285--308},
  year={1993},
  publisher={Elsevier}
}

@article{teixeira2007hybrid,
  title={Hybrid semi-parametric mathematical systems: Bridging the gap between systems biology and process engineering},
  author={Teixeira, Ana P and Carinhas, Nuno and Dias, Jo{\~a}o ML and Cruz, Pedro and Alves, Paula M and Carrondo, Manuel JT and Oliveira, Rui},
  journal={Journal of biotechnology},
  volume={132},
  number={4},
  pages={418--425},
  year={2007},
  publisher={Elsevier}
}

@article{bader2005dimensions,
  title={Dimensions of neural-symbolic integration-a structured survey},
  author={Bader, Sebastian and Hitzler, Pascal},
  journal={arXiv preprint cs/0511042},
  year={2005}
}

@article{hitzler2022neural,
  title={Neural-symbolic learning and reasoning: A survey and interpretation},
  author={Hitzler, P and Sarker, MK and Besold, TR and Garcez, AD and Bader, S and Bowman, H and Domingos, P and Hitzler, P and K{\"u}hnberger, KU and Lamb, LC and others},
  journal={Frontiers in artificial intelligence and applications},
  volume={342},
  pages={1--51},
  year={2022}
}

@article{garcez2019neural,
  title={Neural-symbolic computing: An effective methodology for principled integration of machine learning and reasoning},
  author={Garcez, Artur d'Avila and Gori, Marco and Lamb, Luis C and Serafini, Luciano and Spranger, Michael and Tran, Son N},
  journal={arXiv preprint arXiv:1905.06088},
  year={2019}
}

@article{garcez2023neurosymbolic,
  title={Neurosymbolic ai: The 3 rd wave},
  author={Garcez, Artur d’Avila and Lamb, Luis C},
  journal={Artificial Intelligence Review},
  volume={56},
  number={11},
  pages={12387--12406},
  year={2023},
  publisher={Springer}
}

@inproceedings{belle2015probabilistic,
  title={Probabilistic inference in hybrid domains by weighted model integration},
  author={Belle, Vaishak and Passerini, Andrea and Van den Broeck, Guy},
  booktitle={Proceedings of the Twenty-Fourth International Joint Conference on Artificial Intelligence, IJCAI 2015, Buenos Aires, Argentina, July 25-31, 2015},
  pages={2770--2776},
  year={2015},
  organization={IJCAI Inc}
}

@article{van2021modular,
  title={Modular design patterns for hybrid learning and reasoning systems: a taxonomy, patterns and use cases},
  author={van Bekkum, Michael and de Boer, Maaike and van Harmelen, Frank and Meyer-Vitali, Andr{\'e} and Teije, Annette ten},
  journal={Applied Intelligence},
  volume={51},
  number={9},
  pages={6528--6546},
  year={2021},
  publisher={Springer}
}

@article{marconato2023not,
  title={Not all neuro-symbolic concepts are created equal: Analysis and mitigation of reasoning shortcuts},
  author={Marconato, Emanuele and Teso, Stefano and Vergari, Antonio and Passerini, Andrea},
  journal={Advances in Neural Information Processing Systems},
  volume={36},
  pages={72507--72539},
  year={2023}
}

@article{marra2024statistical,
  title={From statistical relational to neurosymbolic artificial intelligence: A survey},
  author={Marra, Giuseppe and Duman{\v{c}}i{\'c}, Sebastijan and Manhaeve, Robin and De Raedt, Luc},
  journal={Artificial Intelligence},
  volume={328},
  pages={104062},
  year={2024},
  publisher={Elsevier}
}

\appendix

\counterwithin{figure}{section}
\counterwithin{table}{section}

\renewcommand{\thefigure}{\Alph{section}\arabic{figure}}
\renewcommand{\thetable}{\Alph{section}\arabic{table}}


\section{Measurability assumptions for the inference functional}
\label{app:measure}
The functional~\eqref{eq:nesy-functional} is well-defined under the following standard conditions. We assume $(\OmegaSet,\Sigma_{\OmegaSet})$ is measurable with $\Omega'\in\Sigma_{\OmegaSet}$; that $\omega\mapsto\LogicFn(\varphi,\omega)$ is measurable for each query $\varphi$; that the belief induces a finite measure $B_{\boldsymbol{\theta},x}$ on $\Omega'$; and that $\LogicFn(\varphi,\cdot)\in L^1(\Omega',B_{\boldsymbol{\theta},x})$. Dirac beliefs correspond to $B_{\boldsymbol{\theta},x}=\delta_{\omega^\ast(x)}$, recovering pointwise evaluation $F_{\boldsymbol{\theta},x}(\varphi)=\LogicFn(\varphi,\omega^\ast(x))$.
%

\section{Proof of Proposition~\ref{prop:separation}}
\label{app:proof}
\begin{proof}
For (1), the partition $\Omega_{\mathrm{viol}}^{\tau}(\varphi)$ versus
$\Omega_{\mathrm{adm}}^{\tau}(\varphi)$ is defined directly from
$\LogicFn(\varphi,\cdot)$ and $\tau$, hence depends only on $(L,\Sem,\LogicFn)$ and
not on $B_{\boldsymbol{\theta},x}$. For the integral identity in the hard case
$\tau=0$, note that under the hypothesis $\LogicFn(\varphi,\cdot)\ge 0$ the set
$\Omega_{\mathrm{viol}}^{0}(\varphi)=\{\omega\in\Omega':\LogicFn(\varphi,\omega)\le 0\}$
coincides with $\{\omega\in\Omega':\LogicFn(\varphi,\omega)=0\}$. Hence the integrand
vanishes on $\Omega_{\mathrm{viol}}^{0}(\varphi)$, that subset contributes zero, and
\[
F_{\boldsymbol{\theta},x}(\varphi)
=\int_{\Omega'} \LogicFn(\varphi,\omega)\,\mathrm{d}B_{\boldsymbol{\theta},x}(\omega)
=\int_{\Omega_{\mathrm{adm}}^{0}(\varphi)} \LogicFn(\varphi,\omega)\,\mathrm{d}B_{\boldsymbol{\theta},x}(\omega).
\]
For (2), since $\LogicFn(\varphi,\cdot)$ is measurable, substituting
$B_{\boldsymbol{\theta},x}=\delta_{\omega^\ast(x)}$ into
$F_{\boldsymbol{\theta},x}(\varphi)=\int_{\Omega'}\LogicFn(\varphi,\omega)\,\mathrm{d}B_{\boldsymbol{\theta},x}(\omega)$
collapses the integral to $\LogicFn(\varphi,\omega^\ast(x))$. By the definition of
$\Omega_{\mathrm{adm}}^{\tau}(\varphi)$, this value exceeds $\tau$ iff
$\omega^\ast(x)\in\Omega_{\mathrm{adm}}^{\tau}(\varphi)$, which is the stated
admissibility criterion. For non-Dirac beliefs, $F_{\boldsymbol{\theta},x}(\varphi)$
is the $B_{\boldsymbol{\theta},x}$-weighted average of $\LogicFn(\varphi,\cdot)$, so
(in the hard case) its value depends only on how $B_{\boldsymbol{\theta},x}$ allocates
mass within $\Omega_{\mathrm{adm}}^{0}(\varphi)$.
\end{proof}
%

\section{H2N principles (full statements with algorithm)}
\label{app:principles}

\noindent\textbf{P1 (Structure as language).}
Mechanistic equations and modular composition graphs determine the \emph{symbols} and well-formed expressions of $L$.
Unknown closures and latent states are introduced as explicit symbols in $L$
so that they can be quantified over by interpretations.

\noindent\textbf{P2 (Semantics as evaluation).}
The semantics $\Sem$ specifies how formulas are evaluated in an interpretation (Boolean truth, fuzzy degree, etc.).
The logic function $\LogicFn$ is the \emph{operational} object used inside the integral:
it implements admissibility and scoring by selecting or reweighting semantic values,
e.g.\ hard satisfaction indicators or residual-based penalties.

\noindent\textbf{P3 (Learning as belief conditioned on input).}
Learned modules induce a belief component $\Bel$ by producing a non-negative density/weight
$b_{\boldsymbol{\theta},x}(\omega)$ over interpretations $\omega\in\Omega'$ given input $x$.
Deterministic learners correspond to Dirac beliefs concentrated at $\omega^\ast(x)$;
Bayesian/ensemble/energy-based learners induce non-degenerate beliefs.

\noindent\textbf{P4 (Constraints and structures live on the logic side).}
Validity domains, trust regions, hard constraints (e.g.\ conservation, positivity), and regime declarations are encoded as
\emph{structural admissibility} rather than plausibility:
(i) by restricting the integration domain $\Omega'$ to an admissible subset of $\OmegaSet$, and/or
(ii) by enforcing selection/penalties inside $\LogicFn$.
This separates \emph{what is admissible} (logic/structure) from \emph{what is plausible} (belief/learning).

\noindent\textbf{P5 (Factorization mirrors architectural composition).}
When the hybrid design factorizes (serial, parallel, mixture), the NeSy construction should make this explicit by
factorizing $\LogicFn$ and/or $b_{\boldsymbol{\theta},x}$ over corresponding sub-interpretations and sub-formulas.
This yields comparable semantics across patterns and clarifies where uncertainty and structures enter.

Hence, comparing hybrid modeling designs reduces to comparing how they change (i) the language/semantics, (ii) admissibility/scoring via $\LogicFn$, (iii) the unknowns integrated over in $\Omega'$, and (iv) the belief factorisation and dispersion.

\begin{algorithm2e}[t]
\footnotesize
\caption{\textsc{H2N} translation}
\label{alg:h2n}
\KwIn{Hybrid modeling design with mechanistic part $M_\alpha$, learned part(s) $g_\beta$, composition $\operatorname{Comp}$, and structures $\mathcal{C}$}
\KwOut{NeSy tuple $(L,\Sem,\OmegaSet,\Bel)$, logic function $\LogicFn$, and induced functional $F_{\boldsymbol{\theta},x}$}

\textbf{Build the language $L$.}
Encode mechanistic equations, wiring/composition predicates, and rule symbols.
Introduce explicit symbols for unknown closures/residuals/gates/slacks and latent states.

\textbf{Define the interpretation space $\OmegaSet$.}
Let $\omega\in\OmegaSet$ assign values to the unknown entities introduced in $L$
(e.g.\ $\alpha,\psi,z$, gating variables, mixture weights, memberships, noise, slack variables).

\textbf{Choose semantics $\Sem$ and define the logic function $\LogicFn$.}
Fix an evaluation scheme (Boolean, fuzzy, residual-/trajectory-based) and define $\LogicFn(\varphi,\omega)$ accordingly:
\emph{hard} admissibility (indicator selectors), \emph{tolerant} admissibility (thresholded residuals), or
\emph{graded} scoring (penalties/likelihood-like scores).

\textbf{Encode structures $\mathcal{C}$ on the logic side.}
Implement hard feasibility by restricting the measurable integration domain to $\Omega'\subseteq\OmegaSet$,
and/or implement soft feasibility by penalties/selection inside $\LogicFn$. 
(These choices affect $F$ unless one renormalizes/conditions the belief.)

\textbf{Define belief $\Bel$ (conditional on inputs $x$).}
Specify a non-negative density function $b_{\boldsymbol{\theta},x}(\omega)$ w.r.t.\ $\Measure$
induced by the learned components (posterior, likelihood-weighted prior, ensemble mixture).
Dirac beliefs arise as degenerate cases; weights can be normalized when probabilities are required.

\textbf{Assemble inference.}
Instantiate $F_{\boldsymbol{\theta},x}(\varphi)$ by Equation~\eqref{eq:nesy-functional} over $\Omega'$.
If the pattern factorizes (serial, parallel, mixture), make the factorization of $\LogicFn$ and/or $b_{\boldsymbol{\theta},x}$ explicit.
\end{algorithm2e}
%

\section{Monte~Carlo procedure for H2N metrics in the case study}
\label{app:mc-protocol}

With $\OmegaSet=\B^\atoms$ and counting measure, the inference functional reduces to a sum over $\omega$:
$F_{\boldsymbol{\theta},x}(\varphi)=\sum_{\omega}\LogicFn(\varphi,\omega)b_{\boldsymbol{\theta},x}(\omega)$.
NoiseCut yields Dirac belief on the first-layer functions $\hat F_m$ and an independent Bernoulli belief on each second-layer atom $o_K$ with parameter
$p_K = n_1(K)/(n_0(K)+n_1(K))$ (and $p_K=\tfrac12$ when $n_0(K)+n_1(K)=0$).
Under this factorisation:

\paragraph{BD.}
Closed form $\mathrm{BD}=\sum_K p_K(1-p_K)$ because the first-layer is Dirac (zero variance) and the $2^M$ second-layer atoms are independent.

\paragraph{SVR.}
We draw $T=20{,}000$ Monte~Carlo samples $F_O^{(t)}\in\{0,1\}^{2^M}$ by independent Bernoulli sampling per row. With Dirac first-layer, the number of violated training samples decomposes per row:
\[
\mathrm{viol}(\omega^{(t)})=\sum_K \bigl[n_1(K)\Ind{F_O^{(t)}(K)=0}+n_0(K)\Ind{F_O^{(t)}(K)=1}\bigr].
\]
The tolerance enters as the violation budget $\kappa$ inside $\LogicFn_{\mathrm{tol}}=\Ind{\mathrm{viol}\le\kappa}$, so an interpretation is admissible when $\mathrm{viol}\le\kappa$ and violating otherwise. Hence $\mathrm{SVR}=1-\tfrac{1}{T}\sum_t \Ind{\mathrm{viol}(\omega^{(t)})\le\kappa}$ is the Monte~Carlo estimate of the belief mass on violating interpretations. $T=20{,}000$ yields standard errors below $0.005$; we verified stability by replicate runs with different MC seeds. (BD requires no sampling: it is the closed form above.)

\section{Tolerance sweep}
\label{app:tolerance-sweep}

This appendix isolates the role of the tolerance $\tau$ (instantiated here and in the case study of Section~\ref{sec:example} by
the violation budget $\kappa$) as a purely \emph{logic-side} design choice, and
turns the two contrasting tolerance properties of the metrics established in
Section~\ref{sec:eval} into a measurement. SVR is tolerance-dependent (property
(ii) of the SVR): it reads the admissibility partition
$\Omega_{\mathrm{adm}}^{\kappa}/\Omega_{\mathrm{viol}}^{\kappa}$, which $\kappa$
resizes. BD is tolerance-independent (property (ii) of the BD): as a second moment
of the belief $b_{\boldsymbol{\theta}}$ it never references $\kappa$ or the
admissibility partition. Sweeping $\kappa$ at a fixed belief should therefore move
SVR across its full range $[0,1]$ while leaving BD (and the deployed predictor's
accuracy) exactly fixed.

We fix one learned model per seed (struct $[4,4,4,4]$, $10\%$ label noise, $50\%$
training fraction, $10$ seeds) and re-evaluate the H2N metrics under eleven budgets,
with $\kappa/S$ ranging from $0$ to $0.5$. The fit is performed once: neither the
first-layer Dirac factors $\hat F_m$ nor the row-Bernoulli parameters $p_K$ are
refitted between budgets. Only the logic function
$\LogicFn_{\mathrm{tol}}=\Ind{\mathrm{viol}(\omega)\le\kappa}$ changes, so any
variation across the columns of Table~\ref{tab:tol-sweep} is attributable to the
logic side alone.

Two columns are flat by construction. Accuracy is constant at $0.998$ because the
deployed predictor is the per-row threshold of the Bernoulli belief, which does not
reference $\kappa$. BD is constant at $0.783$ because it is a functional of
$b_{\boldsymbol{\theta}}$ alone; its within-seed spread across the eleven budgets is
exactly $0$, so the column mean and every per-seed value coincide. The two
logic-side columns instead move monotonically and are mirror images, since
$\Pr[\mathrm{viol}\le\kappa]=1-\mathrm{SVR}$. At budgets below the noise floor
($\kappa/S\in\{0,0.05\}$) essentially every sampled interpretation exceeds the
budget (the cleanest belief consistent with the mechanistic structure still
misclassifies the $\sim\!10\%$ of flipped training labels) so SVR saturates at
$1.000$ and the admissible mass is $0$. As the budget crosses the noise level near
$\kappa/S=0.10$, SVR collapses ($0.403$ at $\kappa/S=0.10$, $0.160$ at $0.15$,
$0.055$ at $0.20$) and reaches $0$ by $\kappa/S=0.40$, where the budget comfortably
absorbs the label noise and every sampled interpretation is admissible.

This makes the SVR/BD distinction of Section~\ref{sec:eval} concrete and complements
Appendix~\ref{app:SVR-and-BD}: there, fixing $\kappa$ and varying the noise moves BD
while leaving SVR near zero; here, fixing the noise and varying $\kappa$ moves SVR
across its entire range while leaving BD untouched. 

\begin{table}[h]
\centering
\caption{Tolerance dependence of the H2N metrics (struct $[4,4,4,4]$, $10\%$ noise, $50\%$ train). Mean over $10$ seeds. Accuracy and BD are independent of $\kappa$; SVR and the admissible mass $\Pr[\mathrm{viol}\le\kappa]=1-\mathrm{SVR}$ vary monotonically.}
\label{tab:tol-sweep}
\begin{tabular}{lcccc}
\toprule
$\kappa/S$ & Accuracy & SVR & BD & $\Pr[\mathrm{viol}\le\kappa]$ \\
\midrule
$0.00$ & 0.998 & 1.000 & 0.783 & 0.000 \\
$0.05$ & 0.998 & 1.000 & 0.783 & 0.000 \\
$0.10$ & 0.998 & 0.403 & 0.783 & 0.597 \\
$0.15$ & 0.998 & 0.160 & 0.783 & 0.840 \\
$0.20$ & 0.998 & 0.055 & 0.783 & 0.945 \\
$0.25$ & 0.998 & 0.016 & 0.783 & 0.984 \\
$0.30$ & 0.998 & 0.004 & 0.783 & 0.996 \\
$0.35$ & 0.998 & 0.001 & 0.783 & 0.999 \\
$0.40$ & 0.998 & 0.000 & 0.783 & 1.000 \\
$0.45$ & 0.998 & 0.000 & 0.783 & 1.000 \\
$0.50$ & 0.998 & 0.000 & 0.783 & 1.000 \\
\bottomrule
\end{tabular}
\end{table}
%

\section{OOD experiment: full numbers}
\label{app:ood-extra}

This appendix gives the full numbers behind Result~2 of
Section~\ref{sec:example} and discusses the structural out-of-distribution (OOD)
protocol. The mechanistic partition routes each input through $M=4$ first-layer
modules whose joint output indexes one of $2^M=16$ second-layer rows; for the
randomly drawn ground-truth modules all $16$ rows are reachable. We hold
$N_{\mathrm{ood}}\in\{1,2,4,6\}$ of these rows out of training entirely (with
$10\%$ label noise and budget $\kappa=\lfloor 0.10\,S\rfloor$), so that any test
input landing on a held-out row forces a prediction on a region of $\Omega'$ that
received no training support: structural extrapolation rather than
interpolation. Because the row partition is \emph{logical} (fixed by $\Omega'$
rather than by the learned predictor) belief dispersion decomposes exactly as
$\mathrm{BD}=\mathrm{BD}_{\text{seen}}+\mathrm{BD}_{\text{unseen}}$, with
$\mathrm{BD}_{\text{unseen}}=\tfrac14\,|\{K:n_0(K)+n_1(K)=0\}|$ collecting the
maximum-entropy variance ($p_K=\tfrac12$) of the uncovered rows. By
Remark~\ref{rem:coverage} this term is read off the \emph{learned} first-layer
factorisation alone, before any OOD label is observed, which is what makes it a
deployment-time quantity rather than a post-hoc one.

\begin{figure}[h]
\centering
\includegraphics[width=0.6\linewidth]{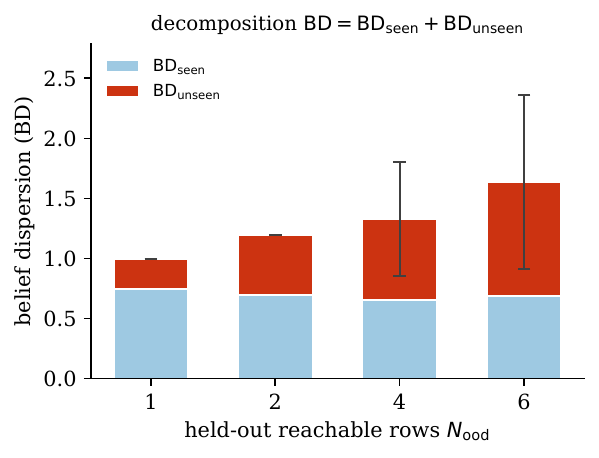}
\caption{Structural OOD shift (struct $[4,4,4,4]$, $10\%$ noise, $10$ seeds). The additive H2N decomposition $\mathrm{BD}=\mathrm{BD}_{\text{seen}}+\mathrm{BD}_{\text{unseen}}$: $\mathrm{BD}_{\text{unseen}}$, computed from row coverage alone, grows monotonically with $N_{\mathrm{ood}}$ and identifies the belief mass on the held-out region without observing OOD test data (whiskers: $\pm 1$ std of $\mathrm{BD}_{\text{unseen}}$).}
\label{fig:ood-appendix}
\end{figure}

\begin{table}[h]
\centering
\caption{Structural OOD shift on second-layer rows (struct $[4,4,4,4]$, $10\%$ label noise, $\kappa=\lfloor 0.10\,S\rfloor$). Mean (std) over $10$ seeds. All $2^M=16$ second-layer rows are reachable for the randomly drawn ground-truth modules; the experiment holds out $N_{\mathrm{ood}}$ of them entirely from training. ``unseen frac.'' is the mean fraction of rows left uncovered under the \emph{learned} factorisation.}
\label{tab:ood-summary}
\small
\setlength{\tabcolsep}{4pt}
\begin{tabular}{lcccccc}
\toprule
$N_{\mathrm{ood}}$ & Acc.\ in-dist & Acc.\ OOD & BD & $\mathrm{BD}_{\text{seen}}$ & $\mathrm{BD}_{\text{unseen}}$ & unseen frac. \\
\midrule
1 & 0.998\,(0.007) & 0.347\,(0.473) & 0.997\,(0.095) & 0.747\,(0.095) & 0.250\,(0.000) & 0.062 \\
2 & 0.998\,(0.007) & 0.345\,(0.282) & 1.199\,(0.100) & 0.699\,(0.100) & 0.500\,(0.000) & 0.125 \\
4 & 0.998\,(0.005) & 0.556\,(0.263) & 1.330\,(0.370) & 0.655\,(0.116) & 0.675\,(0.472) & 0.169 \\
6 & 0.990\,(0.024) & 0.494\,(0.308) & 1.635\,(0.484) & 0.685\,(0.350) & 0.950\,(0.725) & 0.237 \\
\bottomrule
\end{tabular}
\end{table}

As shown in Figure~\ref{fig:ood-appendix} and Table~\ref{tab:ood-summary}, mean $\mathrm{BD}_{\text{unseen}}$ rises monotonically with the hold-out size ($0.25, 0.50, 0.68, 0.95$), and at the smaller hold-outs ($N_{\mathrm{ood}}=1,2$) it is \emph{deterministic} across seeds (std $0$): the learned first-layer partition reliably keeps every held-out row as a distinct uncovered row, so $\mathrm{BD}_{\text{unseen}}=\tfrac14 N_{\mathrm{ood}}$ exactly. At the larger hold-outs ($N_{\mathrm{ood}}=4,6$) a seed-to-seed variance appears (std $0.47$ and $0.73$): in $3$ of the $10$ seeds at each level the learner recovers a \emph{label-equivalent} factorisation that maps held-out ground-truth rows into already-observed learned equivalence classes, giving $\mathrm{BD}_{\text{unseen}}=0$. This is exactly what the decomposition is meant to report: $\mathrm{BD}_{\text{unseen}}$ measures the belief mass on rows uncovered \emph{under the actually-learned factorisation}, the only one accessible at deployment time, not under a hypothetical ground-truth one.

The OOD-accuracy variance has a separate origin. A held-out row receives the maximum-entropy belief $p_K=\tfrac12$, whose thresholded prediction is either right or wrong for that entire row; when few rows are held out (especially $N_{\mathrm{ood}}=1$) the per-seed OOD accuracy is therefore near-bimodal (std $0.47$) and concentrates as more rows are averaged. Despite this variance the qualitative pattern is robust: in-distribution accuracy is essentially unchanged across $N_{\mathrm{ood}}$ while OOD accuracy is far lower, and $\mathrm{BD}_{\text{unseen}}$ (computed before any OOD label is seen) grows with the hold-out. Finally, $\mathrm{BD}_{\text{unseen}}$ is informative about \emph{severity}, not only presence: pooling all seeds, those with $\mathrm{BD}_{\text{unseen}}>0$ average $0.39$ OOD accuracy against $0.69$ for the seeds where the learner absorbed the held-out rows ($\mathrm{BD}_{\text{unseen}}=0$). A zero $\mathrm{BD}_{\text{unseen}}$ thus does not certify OOD robustness (those models still fall well below their in-distribution accuracy) but signals that the model \emph{believes} it has coverage, which is the deployment-time statement the decomposition is designed to make.
%

\section{SVR and BD are complementary}
\label{app:SVR-and-BD}

SVR and BD probe opposite sides of the logic--belief separation
(Proposition~\ref{prop:separation}) and are, by construction, free to vary
independently. SVR is a functional of the admissibility partition
$\Omega_{\mathrm{adm}}^{\tau}/\Omega_{\mathrm{viol}}^{\tau}$, which is a property of $(L,\Sem,\LogicFn)$, weighted by the belief, and it is governed by the tolerance
$\tau$. BD is the trace of the belief covariance and depends on $b_{\boldsymbol{\theta},x}$
alone, never referencing $\LogicFn$ or $\tau$. The two therefore read out distinct
degrees of freedom: moving the tolerance at a fixed belief slides SVR while leaving
BD unchanged (Appendix~\ref{app:tolerance-sweep}), and concentrating or dispersing
the belief at a fixed admissibility moves BD while leaving SVR fixed.

All four combinations are realisable, so neither metric can be recovered from the
other. A model may disperse its belief widely (large BD) yet rarely violate its
constraints (small SVR) when the admissible region is large or the tolerance is
loose; conversely, a model with a sharp belief (small BD) can still place that
belief largely outside the admissible region (large SVR). SVR thus answers whether
the learned belief respects the mechanistic structure, while BD answers how
concentrated the learned plausibility is; neither subsumes the other, and reporting
both separates structural violations from epistemic uncertainty rather than
collapsing them into a single accuracy figure.

The case study (Section~\ref{sec:example}) exhibits this decoupling directly. Under
a fixed tolerance $\kappa=\lfloor 0.5\,S\rfloor$, Table~\ref{tab:noise-sweep} and
Figure~\ref{fig:noise-appendix} show SVR pinned at or near zero across a wide noise
range ($\le 0.011$ through $20\%$ noise, and still only $0.083$ at $30\%$): the
generous budget keeps almost all sampled interpretations admissible. BD, by
contrast, climbs monotonically over the same range, from $0$ at $0\%$ noise to
$3.539$ at $30\%$ and toward its ceiling $2^{M}/4=4$ as the row beliefs approach
maximum entropy. SVR departs from zero only once the noise pushes the realized
violation count past the fixed budget, rising to $0.460$ at $50\%$, by which point
BD has already saturated. The two metrics therefore move on different schedules: BD
tracks the growth of epistemic uncertainty in the learned rows from the very first
noise increment, whereas SVR is a logic-side feasibility statement that fires only
when the chosen tolerance is exceeded.

\begin{figure}[h]
\centering
\includegraphics[width=0.6\linewidth]{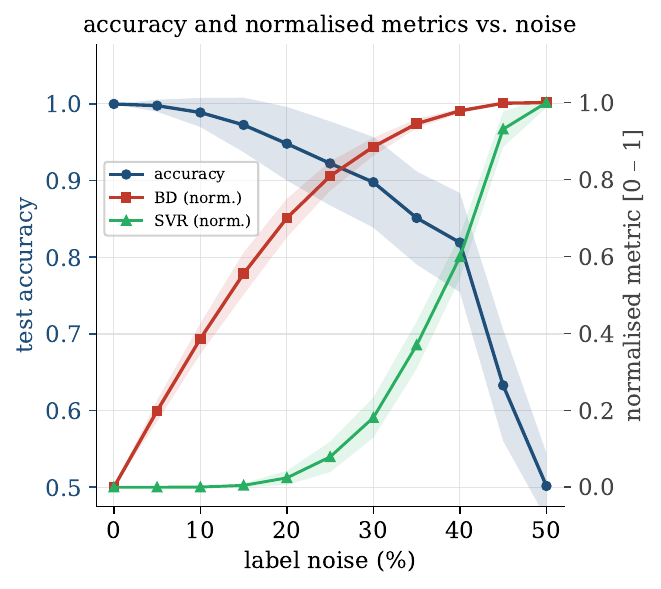}
\caption{Noise sweep, $50$ seeds per level (struct $[4,4,4,4]$, $\kappa=\lfloor 0.5\,S\rfloor$). Mean test accuracy, mean belief dispersion BD, and mean SVR (shaded band: $\pm 1$ standard deviation across seeds) as functions of the label-noise level. The y-axis for the two metrics has been normalized for better visualizations.}
\label{fig:noise-appendix}
\end{figure}

\end{document}